\documentclass{article}
\usepackage{iclr2027_conference, times}

\usepackage{amsmath,amsfonts,bm}

\def\eqref#1{equation~\ref{#1}}

\def\1{\bm{1}}

\DeclareMathAlphabet{\mathsfit}{\encodingdefault}{\sfdefault}{m}{sl}
\SetMathAlphabet{\mathsfit}{bold}{\encodingdefault}{\sfdefault}{bx}{n}

\usepackage{hyperref}              
\usepackage{url}                   
\usepackage{graphicx}              
\usepackage{amsmath, amssymb}      
\usepackage{amsthm}                
\usepackage{algorithm}             
\usepackage{algpseudocode}         
\usepackage{bbm}                   
\usepackage{placeins}              

\theoremstyle{plain}
\newtheorem{theorem}{Theorem}
\newtheorem{lemma}{Lemma}

\newtheorem{conjecture}{Conjecture}

\theoremstyle{definition}
\newtheorem{definition}{Definition}

\newtheorem{remark}{Remark}

\renewenvironment{remark}[1][]
  {\begin{oldremark}[#1]}
  {\nolinebreak\hspace{0.2em}$\square$\end{oldremark}}

\title{Structural Limits of the Information-Theoretic Uncertainty Decomposition}

\author{Jakob Lønborg Christensen \thanks{This work was supported by Danish Data Science Academy, which is funded by the Novo Nordisk Foundation (NNF21SA0069429) and VILLUM FONDEN (40516).} \\
DTU Compute \\
Technical University of Denmark \\
Email: jloch@dtu.dk \\
\And
Christian F. Baumgartner \\
Faculty of Health Sciences and Medicine \\
University of Lucerne \\
Email: christian.baumgartner@unilu.ch \\
\And
Morten Rieger Hannemose \\
DTU Compute \\
Technical University of Denmark \\
Email: mohan@dtu.dk \\
\And
Anders Bjorholm Dahl \\
DTU Compute \\
Technical University of Denmark \\
Email: abda@dtu.dk \\
\AND
Vedrana Andersen Dahl \\
DTU Compute \\
Technical University of Denmark \\
Email: vand@dtu.dk}

\iclrfinalcopy

\begin{document}

\maketitle


\begin{abstract}
    Uncertainty estimation in machine learning typically decomposes uncertainty into aleatoric uncertainty (AU) and epistemic uncertainty (EU) using the standard information-theoretic framework. However, in practice, two critical issues arise: entanglement (AU and EU are highly correlated) and epistemic collapse (EU magnitude shrinks with increasing model capacity). We analyze this framework on a functional level and discover that significant portions of the assumed AU, EU range are infeasible in finite settings, and cannot be attained with any class probabilities. We characterize how this infeasible region scales with the number of classes and Monte Carlo samples $N$ (e.g., from ensembles with $N$ members), revealing it is bounded by $\text{AU} \leq \log(2)/N$. Crucially, the infeasible region's boundary helps explain epistemic collapse: when model confidence is high, $\text{AU} > \text{EU}$ is guaranteed by this fundamental structural limitation. Our findings show that increasing ensemble size mitigates epistemic collapse by reducing the infeasible area. Lastly, we caution against interpreting AU and EU as independent quantities in low AU regimes, since we show they are coupled when $\text{AU} \leq \log(2)/N$. 
\end{abstract}

\FloatBarrier
\section{Introduction}
\label{sec:introduction}
\label{sec:introduction_content}

\begin{figure}[!b]
    \centering
    \includegraphics[width=0.9\textwidth]{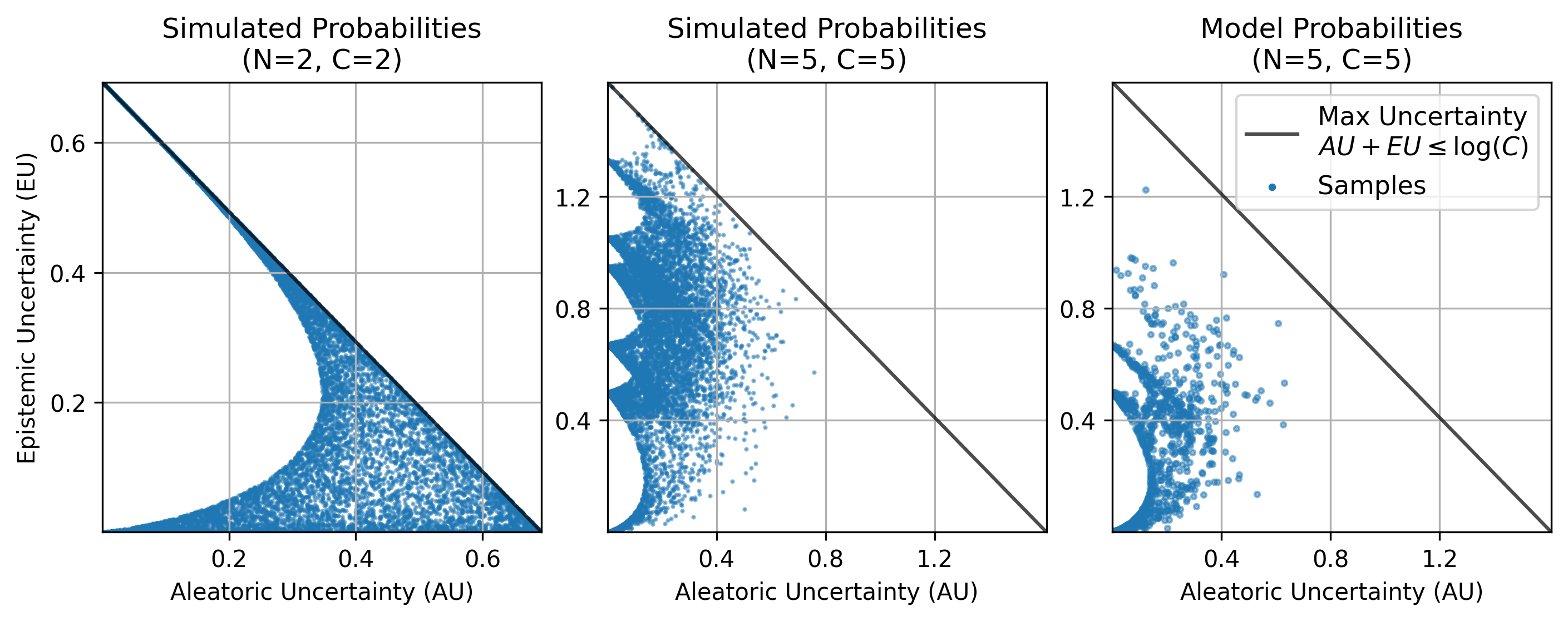}
    \caption{Visualizations of the uncertainty space (Aleatoric Uncertainty vs Epistemic Uncertainty) along with sample uncertainties from probabilities in different settings. The probabilities are simulated in the first two panels, while the last panel shows real probabilities from an ensemble of five models ($N=5$) trained on $C=5$ classes of CIFAR-10.}
    \label{fig:teaser}
\end{figure}

Predictive uncertainty is commonly decomposed into \emph{aleatoric uncertainty} (AU), associated with ambiguity in the data, and \emph{epistemic uncertainty} (EU), associated with limited model knowledge. This distinction is important because it can lead to different decisions in downstream uncertainty estimation tasks (e.g. out-of-distribution detection, calibration, ambiguity modeling).

For a finite set of model predictions, the standard information-theoretic decomposition~\citep{houlsby2011bayesian,kendall_gal,depeweg2018decomposition} defines total uncertainty (TU) as the entropy of the mean predictive distribution, AU as the mean entropy of individual predictions, and EU as the difference between TU and AU. We use \emph{posterior sample} broadly for each predictive distribution in this finite set: it may be generated by an approximate Bayesian method, such as MC dropout or SWAG, or by an ensemble member. Despite its widespread use, studies report strong empirical correlation and entanglement between AU and EU and question whether the measures consistently isolate their intended uncertainty sources~\citep{wimmer2023quantifying,mucsanyi2024benchmarking,kahl2024values,christensen2026rethinking,benchmarking_uncertainty2026}. A related concern is epistemic uncertainty collapse: EU can decrease as model capacity increases, even when limitations in training data or distribution shift persist~\citep{fellaji2024epistemic,kirsch2024implicit}.

We show that the functional dependence of AU and EU partly explains these shortcomings. Figure~\ref{fig:teaser} plots $(\text{AU},\text{EU})$ pairs for synthetic probabilities drawn from correlated Gaussian logits for two scenarios and for predictions from a trained ensemble. The points follow a parametric form that results directly from the information theoretic decomposition. Moreover, it can be seen that depending on the choice of number of classes $C$ and number of posterior samples $N$, certain areas of the plane are \textit{infeasible}. Further details of the simulation are provided in Appendix~\ref{app:logit_simulation}. These observations motivate a systematic analysis of the uncertainty space. 

Our contributions are:

\begin{itemize}
    \item We establish bounds on the attainable $(\text{AU},\text{TU})$ space for finite sets of posterior samples, exposing regions that cannot be realized. Specifically, we show the infeasible region evolves with the posterior sample count $N$ and number of classes $C$, and we derive a $AU\leq\log(2)/N$ upper bound on it.
    \item We support our theoretical analysis, showing it adheres to numerical experiments both trained and simulated. We train ensembles and Monte Carlo dropout models on CIFAR-10 and MNIST. 
    \item Our work shines light on the dynamics behind epistemic collapse. The fraction of epistemic uncertainty grows with $N$, while the epistemic uncertainty fraction shrinks for larger (more confident) models. Both can be explained by the first curve of the infeasible boundary.
\end{itemize}

\section{Background}
\label{sec:background}
\label{sec:background_content}

\subsection{Information-theoretic decomposition}

In classification, the standard approach decomposes predictive (or total) uncertainty (TU) into AU and EU using an information-theoretic identity~\cite{houlsby2011bayesian,depeweg2018decomposition}:

\begin{equation}
\text{AU} = \underbrace{\mathbb{E}_\theta\!\left[ H\!\left(p_\theta\right) \right]}_{\text{\scriptsize Expected Entropy}},
\quad\quad
\text{TU} = \underbrace{H\!\left(\mathbb{E}_\theta\!\left[p_\theta\right]\right)}_{\text{\scriptsize Predictive Entropy}},
\quad\quad
\text{EU} = \underbrace{\mathbb{I}\!\left(Y; \theta \mid x, \mathcal{D}\right)}_{\text{\scriptsize Mutual Information}} = \text{TU} - \text{AU}.
\end{equation}
Here, $H$ denotes Shannon entropy, $\theta$ the model parameters, $p_\theta$ the predictive distribution given parameters $\theta$, $Y$ the label, $x$ the input, and $\mathcal{D}$ the training data.

For finite-sample analysis, we represent the $N$ equally weighted posterior samples for an input as $P=(p_1,\ldots,p_N)\in\mathbb{R}^{N\times C}$, where each $p_i\in\Delta^{C-1}$ is a probability vector on the $C$-class probability simplex. The corresponding empirical mean prediction is $q=\frac{1}{N}\sum_{i=1}^N p_i$, giving $\text{AU}=\frac{1}{N}\sum_{i=1}^N H(p_i)$, $\text{TU}=H(q)$, and $\text{EU}=\text{TU}-\text{AU}$. Equivalently, EU is the mutual information between the uniformly sampled posterior sample index and the predicted class. These finite-sample quantities depend on the predictions in $P$, not on how they were generated.

\subsection{Related Work}
The information-theoretic decomposition has received increasing scrutiny. \citet{wimmer2023quantifying} analyze whether conditional entropy and mutual information have the properties expected of AU and EU. Their theoretical and empirical results identify discrepancies between the measures and the intended uncertainty sources, raising questions about whether the additive decomposition separates AU and EU faithfully.

\citet{mucsanyi2024benchmarking} report strong correlations between estimated AU and EU across uncertainty-estimation methods for image classification. They discuss this association as evidence of limited disentanglement and emphasize task-specific evaluation. For semantic segmentation, \citet{kahl2024values} introduce ValUES, a framework for evaluating uncertainty estimates under ambiguity and distribution shift. Their results show that disentanglement in simulated data does not consistently transfer to real data. EU measures can also capture annotator disagreement and successful disentanglement depends on the data and evaluation setting. Other studies have highlighted similar challenges in disentangling different sources of uncertainty, emphasizing the need for careful interpretation and evaluation of information-theoretic measures \citep{benchmarking_uncertainty2026,christensen2026rethinking}.

A further concern is epistemic uncertainty collapse, in which estimated EU decreases despite persistent limitations in training information. \citet{fellaji2024epistemic} describe an “epistemic uncertainty hole” in which EU becomes small as model size increases, including with limited training data, and examine its effect on out-of-distribution detection. \citet{kirsch2024implicit} proposes implicit ensembling as a possible mechanism: averaging within larger networks may reduce disagreement between ensemble members. The authors provide experimental evidence based on ensembles of ensembles and wider networks that is consistent with this hypothesis.

Closely related mathematical work studies entropy-bounded information maximization. \citet{ebrahimi2024minimum} maximize mutual information between a source variable with a prescribed distribution and an output variable subject to an output-entropy constraint. Taking the source to be a uniformly sampled prediction index and the output to be the predicted class identifies these quantities with EU and TU, respectively. At fixed TU, maximizing EU is equivalent to minimizing AU. They characterize optimal perturbations near deterministic mappings and propose an envelope construction whose global optimality remains conjectural. Our analysis examines the attainable uncertainty region for equally weighted finite sets of posterior samples, including its dependence on sample count and class count, and connects this geometry to epistemic uncertainty collapse.

\section{Theory}
\label{sec:method}
\label{sec:method_content}

\begin{theorem}[Bounds on the range of ($\text{AU}$, $\text{TU}$) pairs]
\label{th:au_tu_bounds}
The range of valid ($\text{AU}$, $\text{TU}$) pairs is bounded by the following inequalities:
\begin{align}
    \text{TU} &\geq \text{AU}, \\
    \text{EU} = \text{TU} - \text{AU} &\leq \log(\min(C, N)), \\
    \text{TU} &\leq \log(C)
\end{align}
\end{theorem}

\textbf{Proof.} $\text{TU} \geq \text{AU}$ follows from the non-negativity of mutual information. EU is the mutual information between posterior sample identity and the predicted label, so it is bounded by the logarithm of the smaller support, $\log(\min(C,N))$. Finally, the entropy of a categorical distribution over $C$ classes is at most $\log(C)$, which gives the third inequality. \qed

\begin{figure}[t]
    \centering
    \includegraphics[width=1.0\textwidth]{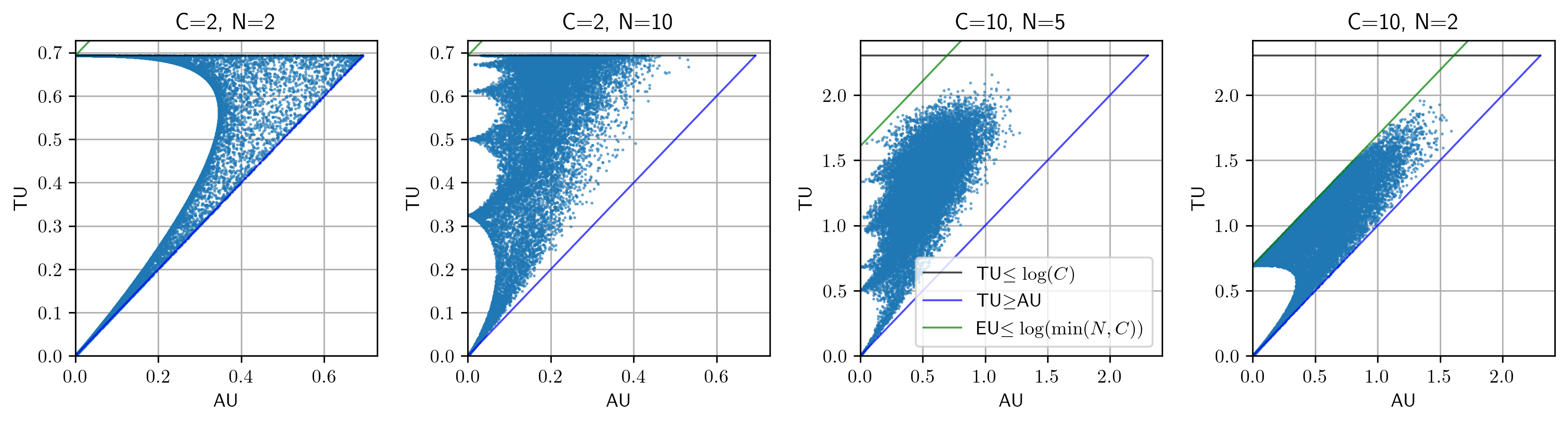}
    \caption{Simulated $(\text{AU},\text{TU})$ values for different posterior sample counts $N$ and class counts $C$, with the bounds from Theorem~\ref{th:au_tu_bounds}.}
    \label{fig:triangles1}
\end{figure}

\begin{theorem}[Valid $\text{TU}$ values when $\text{AU}=0$]
\label{th:au_zero}
When $\text{AU}=0$, the valid $(\text{AU},\text{TU})$ pairs correspond to integer partitions of $N$ into at most $C$ parts. Dividing each partition vector by $N$ gives a possible mean prediction $q$ and hence an attainable TU value. The other zero-AU pairs are infeasible.
\end{theorem}

\textbf{Proof.} If $\text{AU}(P)=\frac{1}{N}\sum_{i=1}^{N}H(p_i)=0$, each prediction $p_i$ is deterministic. There are $C^N$ assignments of $N$ posterior samples to $C$ classes, but TU depends only on the mean prediction $q=\frac{1}{N}\sum_i p_i$. Since entropy is invariant to class permutations, assignments with the same sorted class counts yield the same TU. These count vectors are the integer partitions of $N$ into at most $C$ parts and the zero-AU pairs not generated by such partitions are infeasible. Figure~\ref{fig:triangles2} shows these points. For example, with $N=5$ and $C=3$, the partition $(3,2,0)$ gives $q=(3,2,0)/5=(0.6,0.4,0)$ (three posterior samples predict $(1,0,0)$ and two predict $(0,1,0)$). \qed

To characterize the infeasible region boundary, we need the following definition and conjecture.
\begin{figure}[t]
    \centering
    \includegraphics[width=1.0\textwidth]{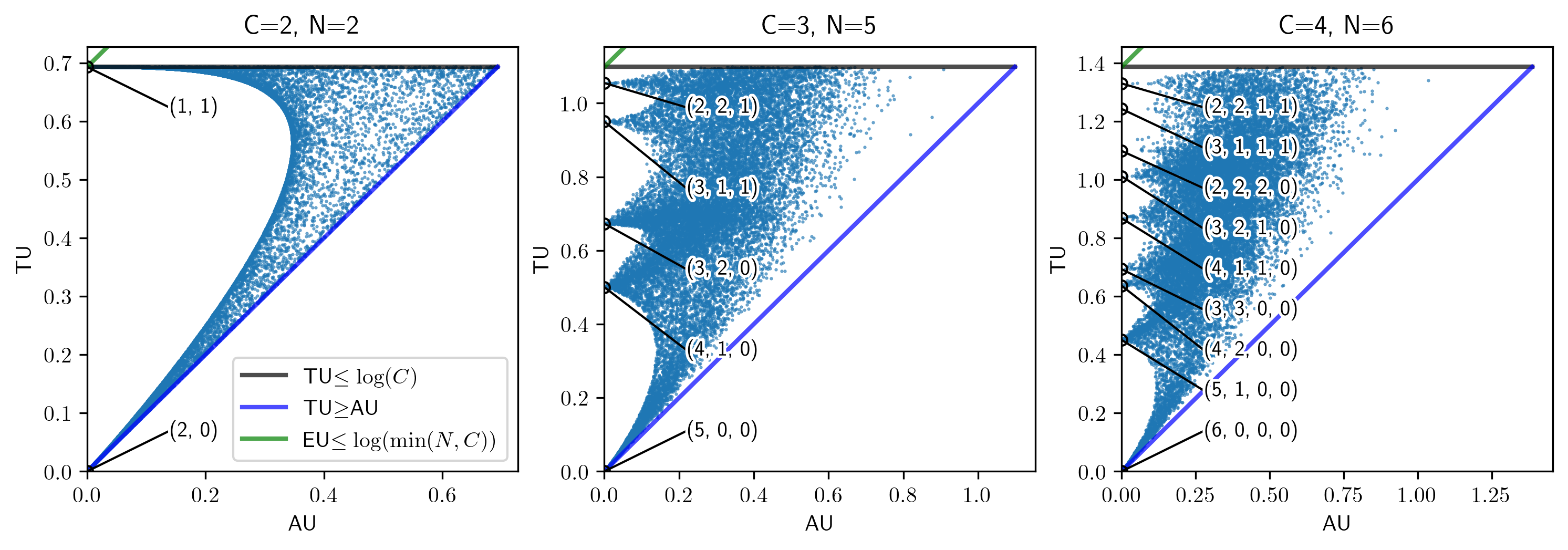}
    \caption{Simulated $(\text{AU},\text{TU})$ values for different posterior sample counts $N$ and class counts $C$. The points with $\text{AU}=0$ are highlighted and labeled by their class counts $Nq$.}
    \label{fig:triangles2}
\end{figure}

\begin{definition}[Deterministic configurations with at most one non-deterministic posterior sample, $\mathcal{D}$]
\label{def:probability_simplex}
We define the set of posterior sample configurations $\mathcal{D}$ such that
\begin{equation}
\begin{aligned}
\mathcal{D} &= \{(p_1, p_2, \ldots, p_N) \in \mathbb{R}^{N \times C} : p_i \in \Delta^{C-1}\  \forall i, \\
     &\bigg(\sum_{i=1}^{N}\sum_{j=1}^{C} \mathbbm{1}(0<p_{i,j} < 1)\bigg) \leq 2 \},
\end{aligned}
\end{equation}
where $\mathbbm{1}(\cdot)$ is the indicator function and $\Delta^{C-1}$ is the $C$-class probability simplex. Thus, $\mathcal{D}$ contains configurations with $N-1$ deterministic posterior samples and at most one posterior sample on an edge between two simplex vertices.
\end{definition}

\begin{conjecture}[Infeasible boundary]
\label{conj:infeasible_boundary}
Let $\text{TU}_{max}^\mathcal{D}=\max_{P \in \mathcal{D}}(\text{TU}(P))$ denote the maximum $\text{TU}$ value attainable by points in $\mathcal{D}$. Then, for a fixed total uncertainty $\text{TU}<\text{TU}_{max}^\mathcal{D}$, only points in $\mathcal{D}$ can minimize the aleatoric uncertainty $\text{AU}$.
\end{conjecture}

For $\text{TU}<\text{TU}_{\max}^{\mathcal{D}}$, the conjecture implies that the infeasible boundary is the lower envelope of the image of $\mathcal{D}$ under the $(\text{AU},\text{TU})$ mapping. This image is connected because $\mathcal{D}$ is connected and the mapping is continuous. It contains $(0,0)$, attained when all posterior samples deterministically predict the same class, and therefore the boundary connects to the origin.

Numerical simulations and empirical results support the conjecture, and it is closely related to the conjectured global optimality of the envelope construction in \citet[Section~3.2]{ebrahimi2024minimum}. We prove it for $C=2$, where the image of $\mathcal{D}$ does not branch in the $(\text{AU},\text{TU})$ plane (Figure~\ref{fig:boundary_shapes}). The general case remains open.

\textbf{Proof (special case).} Let $C=2$ and $N$ be arbitrary. It suffices to show that every $P\notin\mathcal{D}$ is suboptimal for minimizing AU at fixed TU. Such a $P$ has at least two non-deterministic predictions, say those of posterior samples $i$ and $k$. Without loss of generality, let $p_{i,1}=\min(p_{i,1},p_{i,2},p_{k,1},p_{k,2})$. For sufficiently small $\epsilon>0$, transfer probability mass between these predictions as follows:
\begin{equation}
    \tilde{P}_{[i,k],[1,2]} =
    \begin{pmatrix}
    \tilde{p}_{i,1} & \tilde{p}_{i,2} \\
    \tilde{p}_{k,1} & \tilde{p}_{k,2}
    \end{pmatrix}
    =
    \begin{pmatrix}
    p_{i,1} & p_{i,2} \\
    p_{k,1} & p_{k,2}
    \end{pmatrix}
    +
    \begin{pmatrix}
    -\epsilon & \epsilon \\
    \epsilon & -\epsilon
    \end{pmatrix},
\end{equation}
The column sums are unchanged, so $q$ and TU remain fixed. Let $h$ denote binary entropy. The perturbation changes the two relevant entropy terms to $h(p_{i,1}-\epsilon)+h(p_{k,1}+\epsilon)$, whose derivative at $\epsilon=0$ is $-h'(p_{i,1})+h'(p_{k,1})$. Since $p_{i,1}\leq p_{k,1}$ and $h'$ is decreasing, this derivative is non-positive.
Strict concavity then implies that a sufficiently small positive perturbation decreases the summed entropy (Figure~\ref{fig:binary_entropy_shift}). The prediction space is compact and the $(\text{AU},\text{TU})$ mapping is continuous, so a minimizer exists and it lies in $\mathcal{D}$. \qed

\begin{figure}[ht]
\centering
\includegraphics[width=0.4\textwidth]{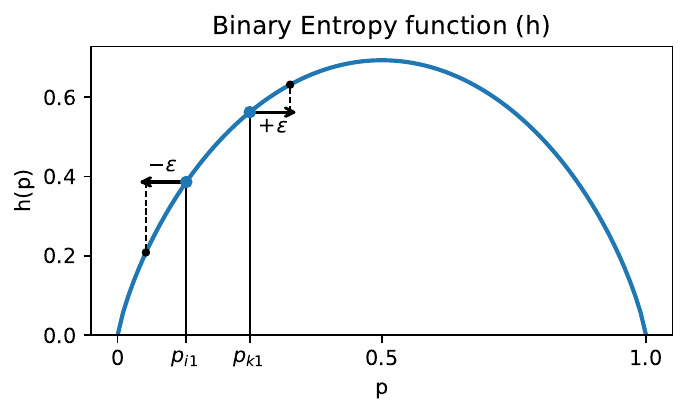}
\caption{Binary entropy at the original and perturbed prediction probabilities.}
\label{fig:binary_entropy_shift}
\end{figure}

\begin{figure}[t]
    \centering
    \includegraphics[width=1.0\textwidth]{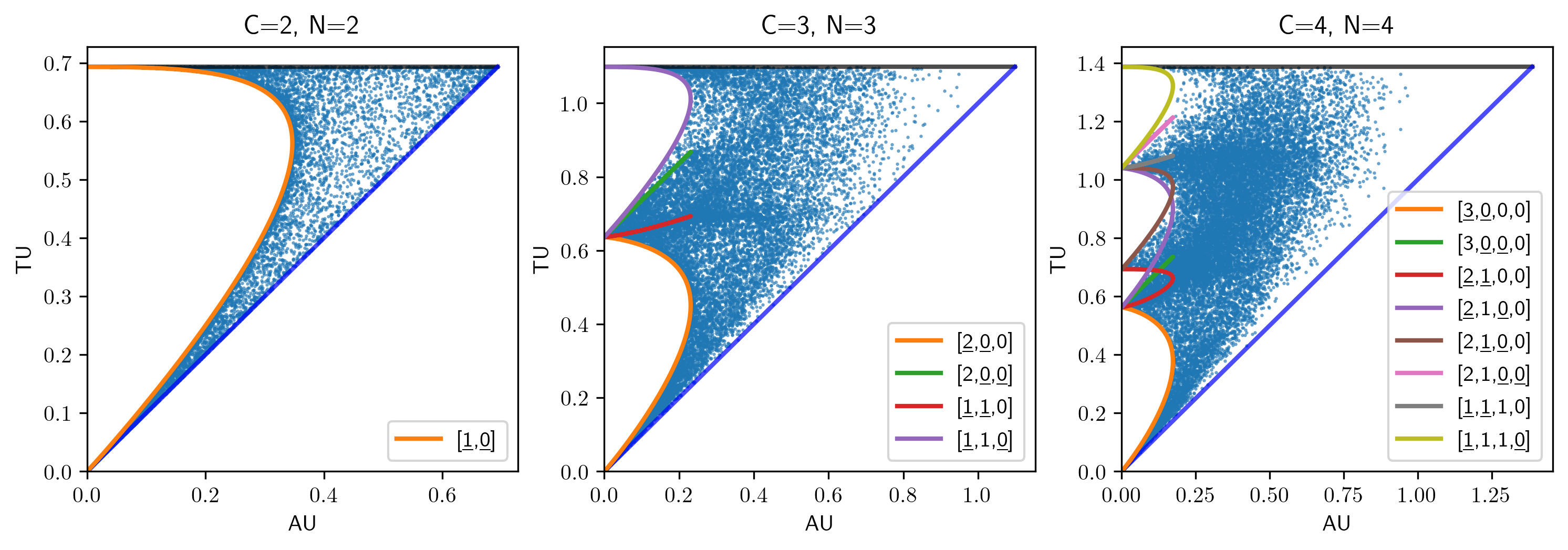}
    \caption{Simulated $(\text{AU},\text{TU})$ values for different posterior sample counts $N$ and class counts $C$. Curves from $\mathcal{D}$ are labeled by deterministic class counts. The underlined entries mark the two classes supported by the non-deterministic posterior sample. For example, $[\underline{2},\underline{0},0]$ connects the count vectors $[3,0,0]$ and $[2,1,0]$.}
    \label{fig:triangles3}
\end{figure}

\begin{figure}[ht]
\centering
\includegraphics[width=1.0\textwidth]{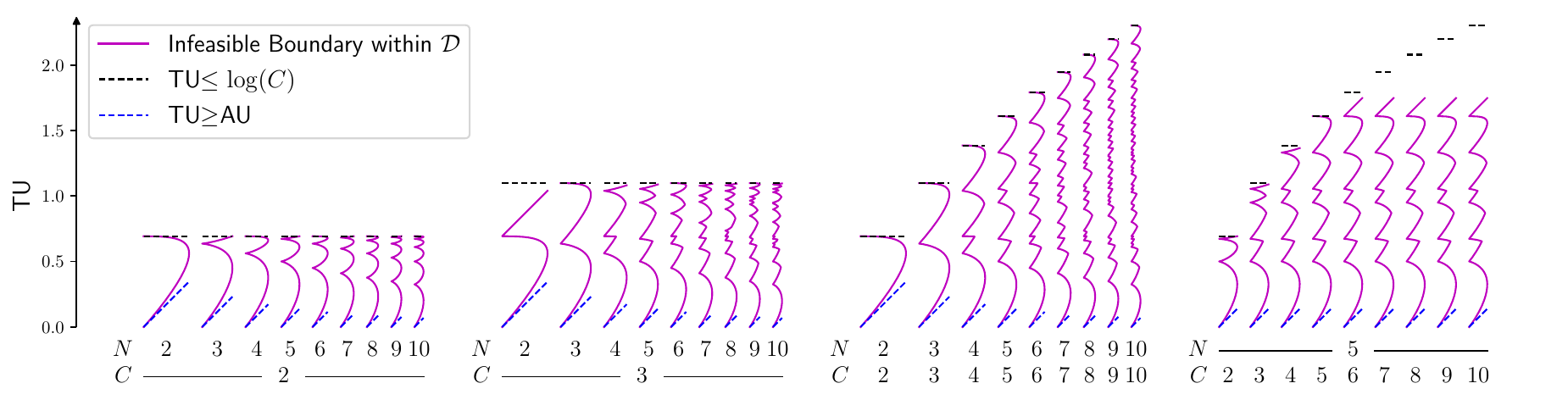}
\caption{Infeasible-boundary shapes, shown as lower envelopes of interpolating curves, for different posterior sample counts $N$ and class counts $C$.}
\label{fig:boundary_shapes}
\end{figure}


Under the conjecture, the infeasible boundary is the lower envelope of curves obtained by interpolating between zero-AU points. Each point from the curves/simplex edges corresponds to a posterior sample configuration in $\mathcal{D}$, and the infeasible region lies below this envelope (Figure~\ref{fig:triangles3}). To be precise, we define the \emph{infeasible region} as the set pairs of $(\text{AU},\text{TU})$ or $(\text{AU},\text{EU})$ that cannot be realized for a finite set of posterior samples but can be realized in the limit of an infinite number of posterior samples.

As an example of the conjecture, consider the base case where $N=C=2$. The $\text{AU}=0$ points are $(0,0)$ and $(0,\log 2)$, with mean predictions $q=(1,0)$ and $q=(1/2,1/2)$. Varying one posterior sample's prediction between the two classes gives $p_1=(1,0)$, $p_2=(1-s,s)$, and $q=(1-s/2,s/2)$ with $s\in[0,1]$. The corresponding uncertainties are
\begin{align} \label{eq:au_tu_s}
    \text{AU}(s) &= \frac{1}{2} H([1,0]) + \frac{1}{2} H([1-s, s]) =\frac{1}{2} H([1-s, s]) \nonumber \\
                 &=   -\frac{1}{2} \left((1-s) \log(1-s) +  s \log(s)\right), \\
    \text{TU}(s) &= H(q) = H([1-s/2, s/2]) = - (1-s/2) \log(1-s/2) - (s/2) \log(s/2).
\end{align}
The resulting curve is shown in Figure~\ref{fig:triangles3} (left). This example only has one boundary curve, but the number of boundary curves increases with $N$ and $C$. The expression for $\text{AU}$ generalizes with $1/2$ replaced by $1/N$, which we state in the following theorem.

\begin{theorem}[AU component of simplex edges]
\label{thm:au_infeasible_boundary}
The $\text{AU}$ component of the simplex edges can be expressed as
\begin{equation}
\text{AU}(s) = -\frac{1}{N} \left((1-s)\log(1-s) + s \log(s)\right), \quad
\end{equation}
and therefore $\text{AU}\leq\log(2)/N$ for all simplex edges.
\end{theorem}

\textbf{Proof.} The only change with respect to the AU calculation in \eqref{eq:au_tu_s} is the replacement of $1/2$ by $1/N$. The maximum of $\text{AU}(s)$ occurs at $s=1/2$, giving $\text{AU}_{\max} = -\frac{1}{N} \left((1/2)\ln(1/2) + (1/2) \ln(1/2)\right) = \frac{\ln(2)}{N}$, which establishes the upper bound. \qed

\begin{theorem}[Lowest $\text{TU}$ boundary curve] 
\label{thm:first_bubble}
The curve connecting $Nq^{(N)}=[N,0,0,\ldots]$ and $Nq^{(N-1)}=[N-1,1,0,\ldots]$ attains the lowest TU values among the infeasible-boundary curves. Its interior TU values are not attained by any other such curve. We call it the first bubble.
\end{theorem}

\textbf{Proof.} Each boundary curve connects two fully deterministic configurations. From $q^{(N)}$, the smallest increase in TU occurs when one posterior sample changes its prediction to a second class, yielding $q^{(N-1)}$. Thus, $\text{TU}(q^{(N)})<\text{TU}(q^{(N-1)})<\text{TU}(\hat q)$ for any other deterministic configuration $\hat q$. The lemma of Appendix~\ref{app:lemma} shows that TU attains its minimum along each curve at an endpoint. Therefore, every other curve has minimum a TU at of least $\text{TU}(q^{(N-1)})$. \qed


\section{Experiments}
\label{sec:experiments}
\label{sec:experiments_content}

\subsection{Experimental Setup}
To evaluate the impact of ensemble size and network capacity on uncertainty estimation, we conduct a series of experiments varying these factors systematically. Our aim is to understand how these factors relate to epistemic collapse, i.e., the reduction of epistemic uncertainty, commonly observed when model capacity grows. 

Our theory suggests that when AU is low, there is a direct dependence between AU and EU since some pairs of (AU,EU) are infeasible. This dependence could relate to the observed entanglement in uncertainty estimation. For this, we track performance for out-of-distribution (OOD) detection as it allows us to assess the entanglement between AU and EU in practice. For OOD detection, we expect EU to be more informative than AU, and we therefore say that there is entanglement when AU performance exceeds EU performance. For OOD detection, we designate the final three classes of each dataset as OOD and the remaining classes as in-distribution (ID). We use EU to score OOD examples and report the area under the receiver operating characteristic curve (AUROC).

We evaluate on CIFAR-10~\citep{cifar10} and MNIST~\citep{mnist} using EfficientNet~\citep{efficientnet} models B0 (4.02M parameters), B2 (7.72M), B4 (17.57M), and B6 (40.76M). We estimate uncertainty using either deep ensembles~\citep{lakshminarayanan2017deepensembles} or Monte Carlo (MC) dropout~\citep{gal2016dropout}. For MC dropout, we enable dropout at test time and insert layers with rate $p=0.2$ before each SiLU activation in EfficientNet. All results are shown on validation data using standard splits.

We train all models from scratch using AdamW~\citep{loshchilov2019adamw} without weight decay, gradient clipping at 1.0, batch size 128, and an initial learning rate of 0.001 with cosine decay. Training lasts 200--1200 epochs, depending on model capacity and dataset. These durations were selected based on validation-performance plateaus observed in preliminary experiments.

\subsection{Adherence to the Theoretical Predictions}

We first test whether posterior samples obtained from trained models satisfy the predicted constraints. Figure~\ref{fig:scatter_10C} shows that the predictions respect the infeasible boundary in the $(\text{AU},\text{EU})$ plane. Many examples also lie near the origin, indicating low estimated uncertainty in both components.

\begin{figure}[t]
    \centering
    \includegraphics[width=0.5\textwidth]{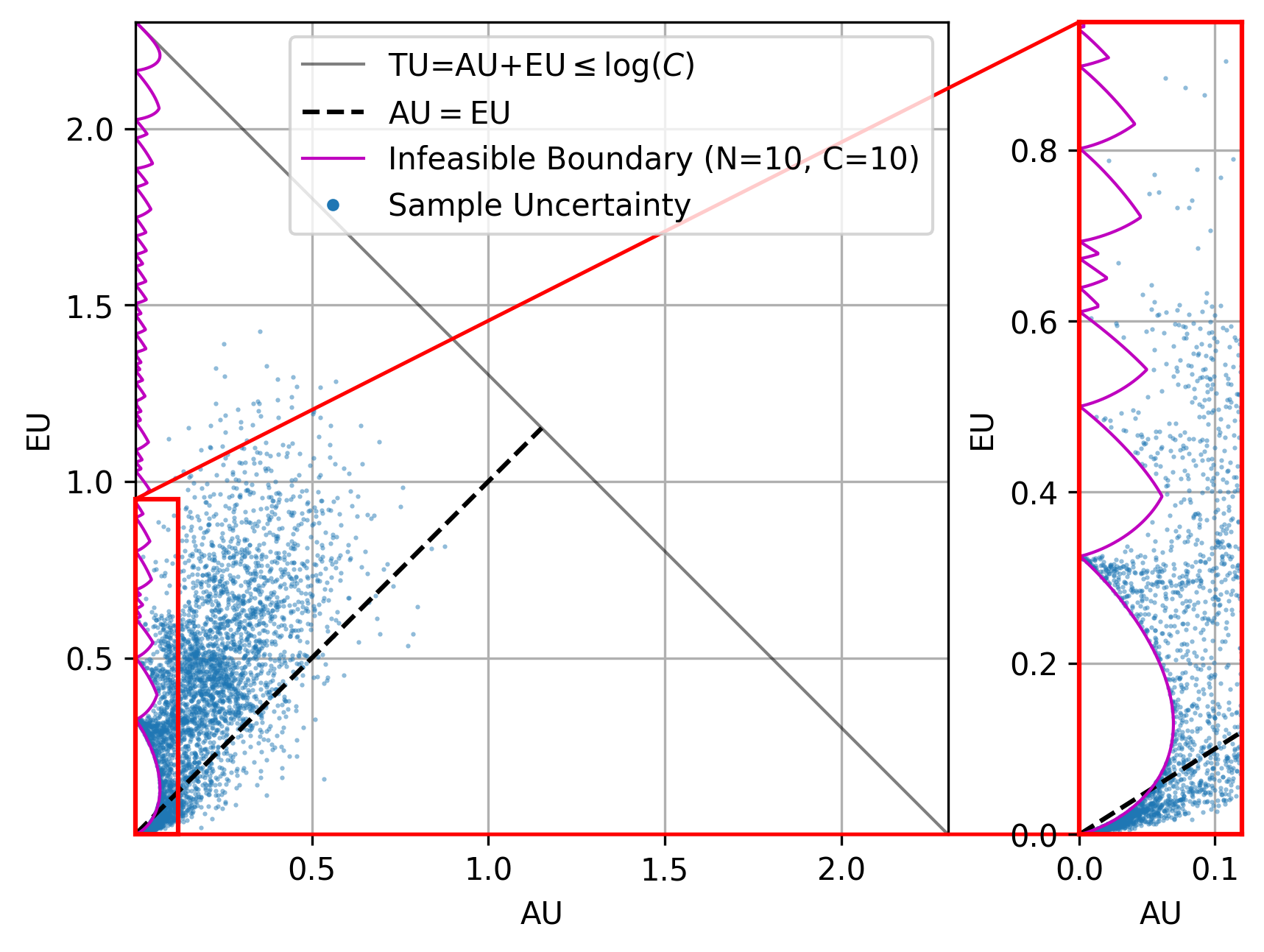}
    \caption{Visualizations of the ($\text{AU}$, $\text{EU}$) space along with points from an ensemble of 10 EfficientNet-B0 networks trained on CIFAR-10.}
    \label{fig:scatter_10C}
\end{figure}

\subsection{Varying the Posterior Sample Count}

Increasing the posterior sample count ($N$) improves OOD detection performance (Figure~\ref{fig:perf_ENS50}) for both AU and EU, but no clear trend is observed in how entangled the measures are. The relative EU fraction, $\text{EU}/\text{TU}$, also increases with $N$, counteracting the effects of epistemic collapse. This trend must be interpreted in light of finite-sample bias: for independent predictions sampled from a fixed distribution over networks, the plug-in EU estimate is downward biased, with bias decreasing as $N$ increases. We therefore report jackknife-corrected estimates in Figure~\ref{fig:perf_ENS50}~\citep{jackknife}. Appendix~\ref{app:jackknife} describes the correction and its limitations.

Figure~\ref{fig:scatter_vs_N} shows that fewer examples concentrate near the origin as $N$ increases. By Theorem~\ref{thm:au_infeasible_boundary}, AU along the infeasible boundary is bounded by $\log(2)/N$. Thus, the boundary's extent in the AU direction decreases as $N$ increases.

\begin{figure}[t]
    \centering
    \includegraphics[width=0.9\textwidth]{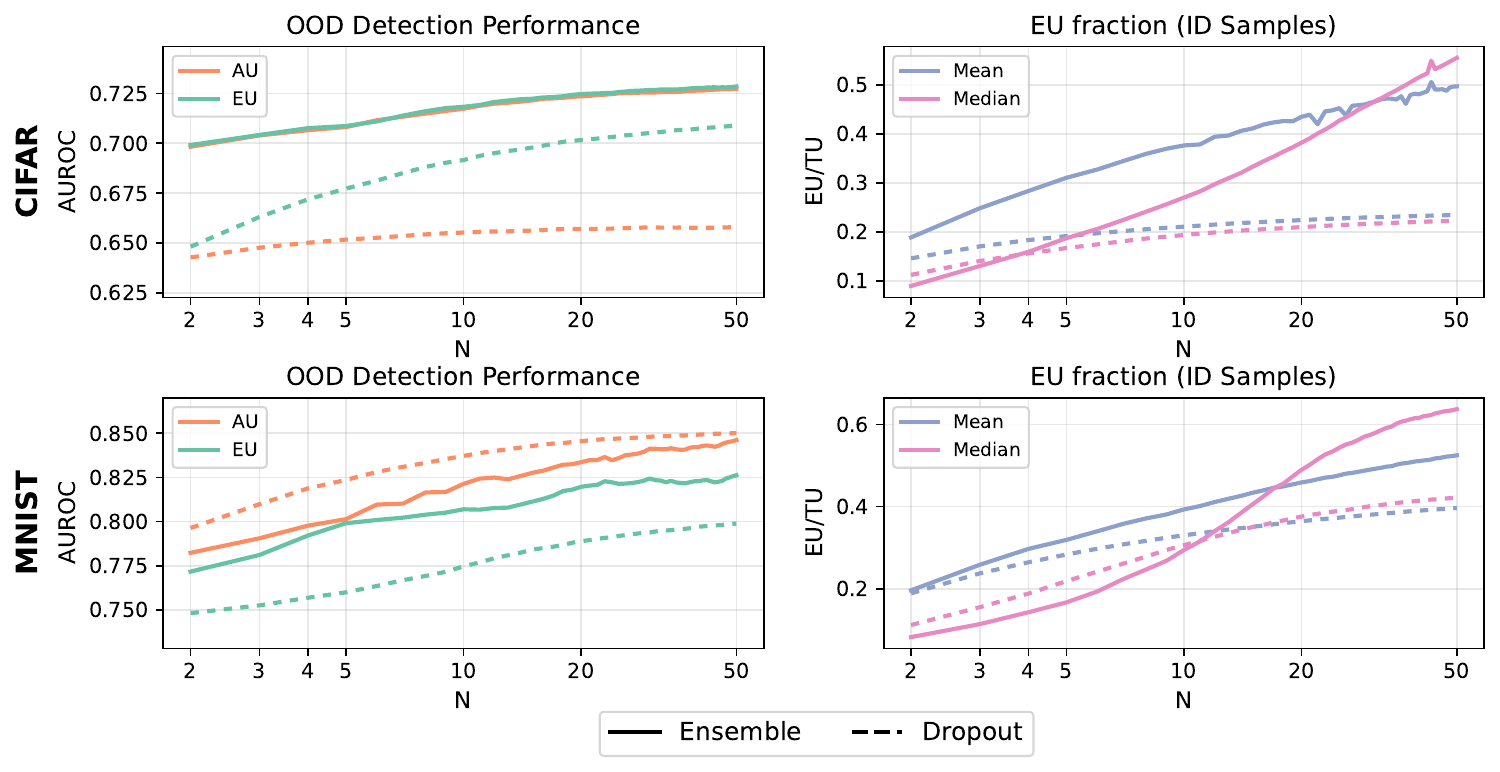}
    \caption{OOD detection performance and epistemic collapse as the number of Enet-B0 networks (N) is varied.}
    \label{fig:perf_ENS50}
\end{figure}

\begin{figure}[t]
    \centering
    \includegraphics[width=1.0\textwidth]{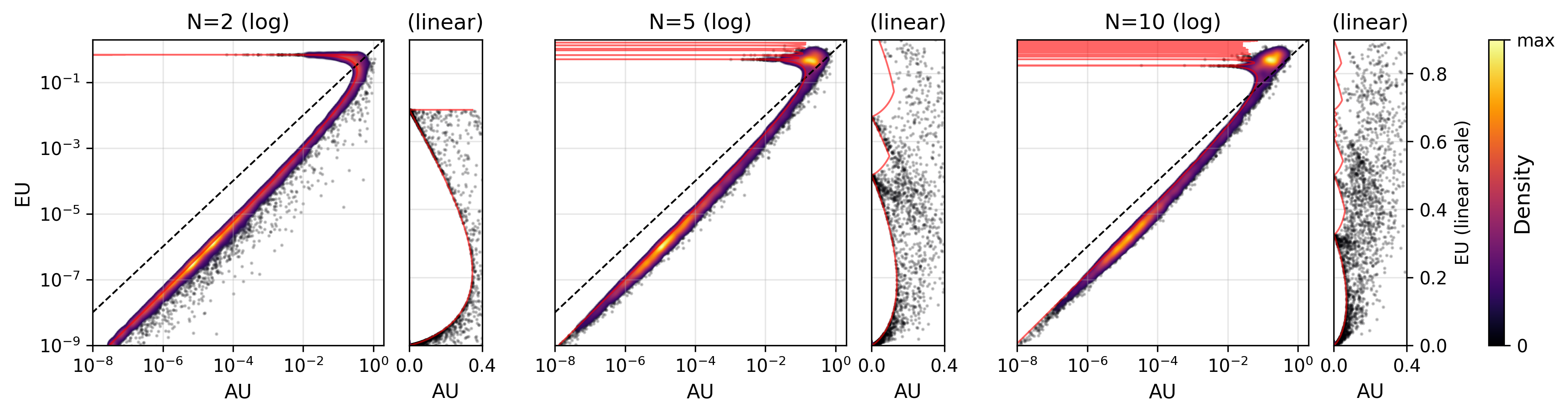} 
    \caption{Visualizations of the ($\text{AU}$, $\text{EU}$) space along with points from an Enet-B4 ensemble trained on CIFAR as the number of networks (N) is varied. The infeasible boundary is shown in red, and the dashed line is $\text{AU}=\text{EU}$.}
    \label{fig:scatter_vs_N}
\end{figure}

\subsection{Varying the Network Capacity}
As model capacity increases, both AU and EU decrease, while accuracy and OOD detection performance improve (Figure~\ref{fig:perf_BASE}). The relative EU fraction, $\text{EU}/\text{TU}$, also declines, reaching approximately $20\%$ for the largest models. This trend is consistent with reports of increased epistemic collapse at larger model sizes~\citep{fellaji2024epistemic}. Entanglement is also clearly visible as AU often exceeds EU in OOD detection performance.

Figure~\ref{fig:scatter_vs_b} shows the $(\text{AU},\text{EU})$ distributions on logarithmic axes. As capacity increases, more examples approach the origin. The first-bubble boundary (Theorem~\ref{thm:first_bubble}) constrains the attainable values in this region, ensuring $\text{AU}>\text{EU}$. The highly confident points therefore attain low EU relative to AU.

\begin{figure}[t]
    \centering
    \includegraphics[width=0.8\textwidth]{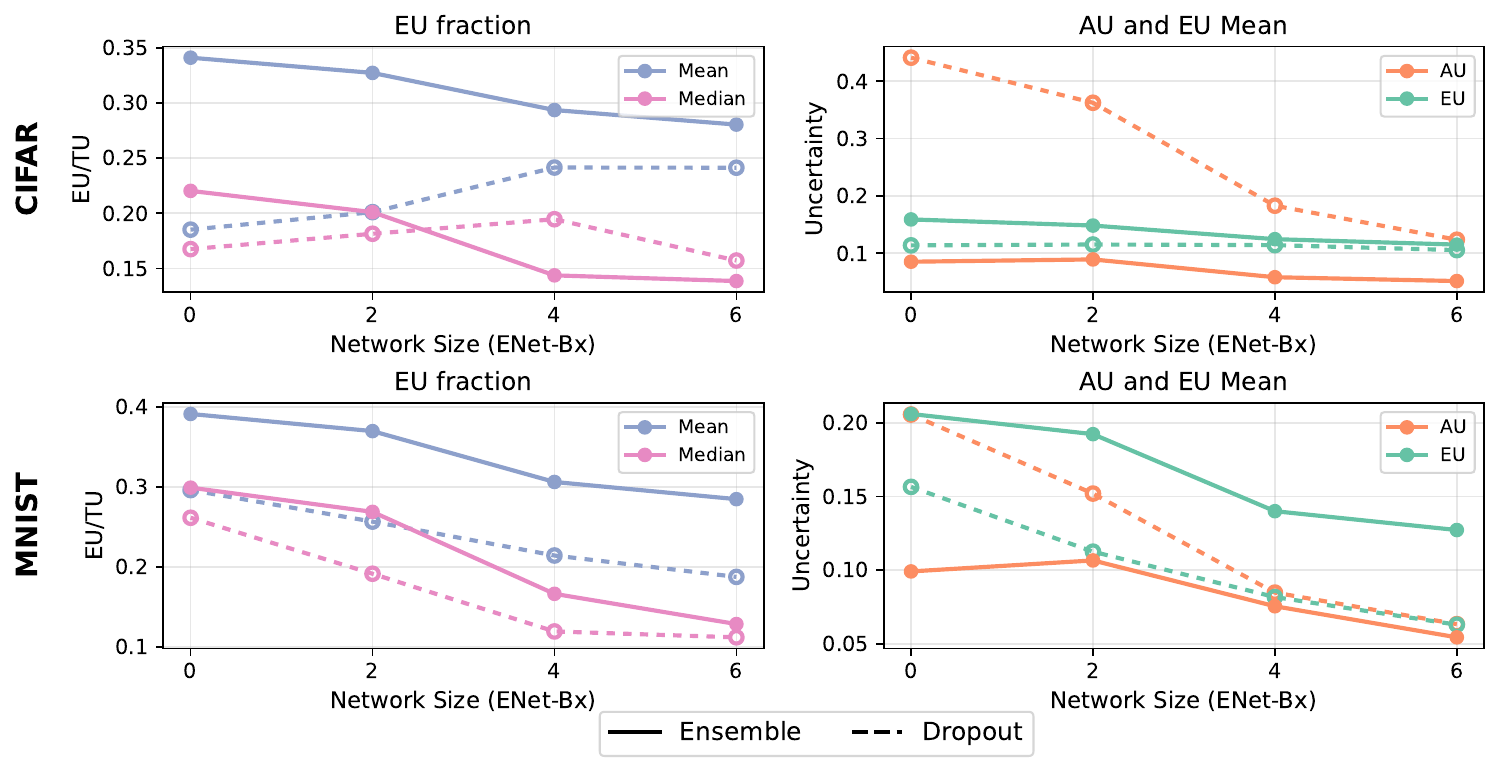}
    \caption{OOD detection performance and epistemic collapse as the network capacity is varied.}
    \label{fig:perf_BASE}
\end{figure}
\begin{figure}[t]
    \centering
    \includegraphics[width=1.0\textwidth]{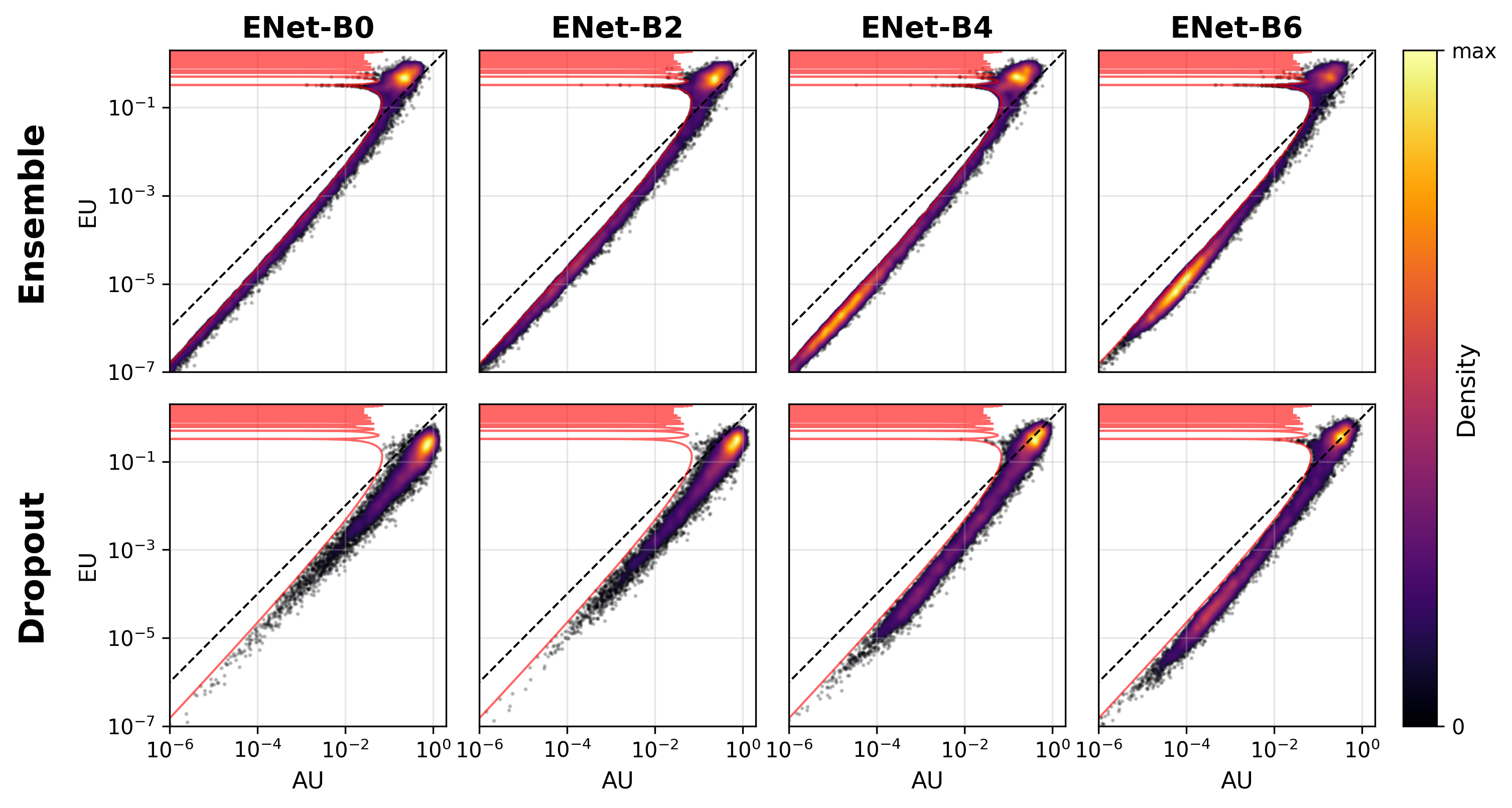}
    \caption{Visualizations of the ($\text{AU}$, $\text{EU}$) space along with points from an ensemble/dropout model as the network capacity is varied. The model is trained on CIFAR. The infeasible boundary is shown in red, and the dashed line is $\text{AU}=\text{EU}$.}
    \label{fig:scatter_vs_b}
\end{figure}

\section{Discussion}
\label{sec:discussion}
\label{sec:discussion_content}

The information-theoretic decomposition imposes structural constraints on the AU and EU pairs attainable from a finite set of posterior samples. Beyond the broad entropy bounds, an infeasible region remains in low-AU regimes. The first boundary curve severely impacts attainable combinations, restricting them to $\text{AU}>\text{EU}$ in this regime. Thus, AU and EU cannot be interpreted as independently varying coordinates for finite-sample predictions here.

These constraints provide one perspective on epistemic collapse. As is well-documented in the literature, larger models produce more highly confident predictions~\cite{guo2017calibration}. Such predictions accumulate near the first infeasible boundary, where the geometry restricts the relative sizes of AU and EU. The boundary therefore offers a structural contribution to the observed decrease in the EU fraction as model capacity grows, although it is not the only cause of collapse. In our experiments, the number posterior samples ($N$) in the ensemble also plays a role: increasing the ensemble size reduces epistemic collapse. Our derived bound $\text{AU}\leq\log(2)/N$ (Theorem~\ref{thm:au_infeasible_boundary}) helps explain why. More generally, increasing the posterior sample count can lessen this particular geometric limitation, although it also incurs additional computational cost.

Our analysis is limited to the information-theoretic decomposition and does not generalize to other frameworks such as EPKL (Expected Pairwise Kullback-Leibler)~\citep{schweighofer2023epkl}. We do not provide a method to eliminate the boundary effect at fixed posterior sample count. Future work could investigate alternative uncertainty frameworks, develop methods for interpreting or correcting finite-sample estimates, and test whether tiered strategies (e.g., ensembles of SWAG (Stochastic Weight Averaging Gaussian)~\citep{maddox2019swag} or dropout models~\citep{gal2016dropout}) can approach the performance of larger ensembles at lower computational cost.

\section{Conclusion}
\label{sec:conclusion}
\label{sec:conclusion_content}

We characterize structural limits on the information-theoretic decomposition of predictive uncertainty for finite sets of posterior samples. The attainable AU--EU region has an infeasible boundary that depends on posterior sample count and class count. Simulations and experiments on CIFAR-10 and MNIST show how this geometry relates to observed uncertainty patterns, including changes in epistemic uncertainty with model capacity. These findings motivate interpreting information-theoretic estimates in light of finite-sample constraints. Extending the boundary analysis beyond the binary-class proof and developing practical mitigation strategies remain open directions.

\newpage
\FloatBarrier
\bibliographystyle{iclr2027_conference}
\bibliography{iclr2027_conference}

\newpage

\appendix
\section{Logit Simulation}
\label{app:logit_simulation}

We generate logits with Gaussian noise and introduce dependence across posterior samples and classes using correlation parameters $\rho_N$ and $\rho_C$. The resulting probabilities illustrate how the attainable uncertainty region changes with the simulation settings.

\begin{algorithm}[h]
\caption{Simulate Probabilistic Predictions}
\label{alg:simulate_probs}
\begin{algorithmic}
    \Procedure{SimulateProbs}{$N, C, n, \sigma_L, \rho_N, \rho_C$}
        \State \Comment{Initialize logits with independent Gaussian noise}
        \State $\mathbf{L} \gets \mathcal{N}(0, \sigma_L^2) \in \mathbb{R}^{N \times C \times n}$
        \State \Comment{Apply posterior sample-wise correlation}
        \State $\mathbf{L} \gets \rho_N \cdot \mathcal{N}(0, \sigma_L^2)_{1 \times C \times n} + (1 - \rho_N) \cdot \mathbf{L}$
        \State \Comment{Apply class-wise correlation}
        \State $\mathbf{L} \gets \rho_C \cdot \mathcal{N}(0, \sigma_L^2)_{N \times 1 \times n} + (1 - \rho_C) \cdot \mathbf{L}$
        \State \Comment{Convert logits to probabilities}
        \State $\mathbf{P} \gets \text{softmax}(\mathbf{L}, \text{axis}=2)$
        \State \Return $\mathbf{P}$
    \EndProcedure
\end{algorithmic}
\end{algorithm}

\begin{figure}[ht]
    \centering
    \includegraphics[width=0.9\textwidth]{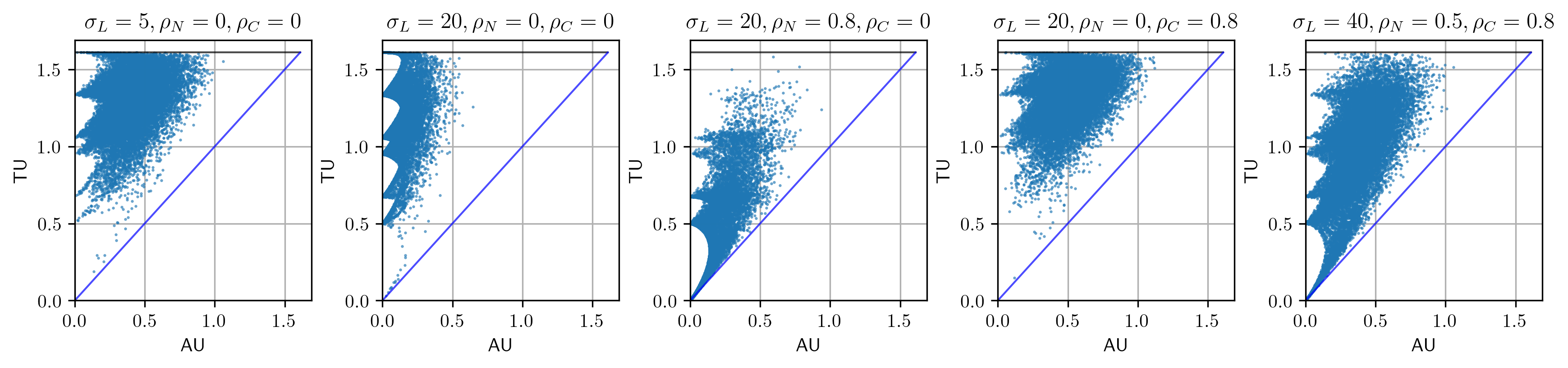}
    \caption{Simulated $(\text{AU},\text{TU})$ values for different settings with $C=N=5$.}
    \label{fig:triangles0}
\end{figure}

\section{Finite-Sample Bias and Jackknife Correction}
\label{app:jackknife}

The finite-sample estimator of epistemic uncertainty (EU) is biased downward because it uses the entropy of the sample-mean prediction. This appendix characterizes that bias and describes the jackknife correction used in our experiments.

For a fixed input, let $p_1,\ldots,p_N$ be independent posterior samples from the same predictive distribution, and define
\begin{equation}
q_N=\frac{1}{N}\sum_{i=1}^{N}p_i,
\qquad
q_\infty=\mathbb{E}_{\theta}[p_\theta].
\end{equation}
The usual estimator can be written as an average divergence from the sample mean:
\begin{equation}
\widehat{\mathrm{EU}}_N
=
H(q_N)-\frac{1}{N}\sum_{i=1}^{N}H(p_i)
=
\frac{1}{N}\sum_{i=1}^{N}
D_{\mathrm{KL}}(p_i\Vert q_N).
\end{equation}
Its population counterpart is
$\mathrm{EU}_\infty
=H(q_\infty)-\mathbb{E}_{\theta}[H(p_\theta)]$.
Although the sample mean $q_N$ and the AU estimator are unbiased, the concavity of entropy makes $H(q_N)$, and consequently the EU estimator, downward biased. In particular,
\begin{equation}
\mathrm{EU}_\infty
-
\mathbb{E}[\widehat{\mathrm{EU}}_N]
=
\mathbb{E}\!\left[
    D_{\mathrm{KL}}(q_N\Vert q_\infty)
\right]
\geq 0,
\end{equation}
where expectations are over posterior samples. Under these assumptions, the expected EU estimate is nondecreasing with $N$ and converges to $\mathrm{EU}_\infty$. Thus, EU can increase with posterior sample count even when the underlying predictive distribution remains fixed.

Unlike squared Euclidean distance, KL divergence does not admit a universal multiplicative correction of the form $N/(N-1)$. We instead use the jackknife bias correction of \citet{jackknife}. For $N\geq2$, define the leave-one-out estimates
\begin{equation}
q_{-i}=\frac{1}{N-1}\sum_{j\neq i}p_j,
\qquad
\widehat{\mathrm{EU}}_{-i}
=
H(q_{-i})
-
\frac{1}{N-1}\sum_{j\neq i}H(p_j).
\end{equation}
The corrected estimator is
\begin{equation}
\widehat{\mathrm{EU}}_{\mathrm{JK}}
=
N\widehat{\mathrm{EU}}_N
-
\frac{N-1}{N}
\sum_{i=1}^{N}\widehat{\mathrm{EU}}_{-i}.
\end{equation}
Each leave-one-out estimate measures sensitivity to sample size. If the expected uncorrected estimator has a smooth expansion in powers of $1/N$, its leading bias term is proportional to $1/N$. The jackknife combination cancels this term because
$N(1/N)-(N-1)(1/(N-1))=0$.
Applying the same construction to the uncorrected sample variance recovers the usual $1/(N-1)$ correction exactly.

For EU, the AU terms in the leave-one-out average equal the full-sample AU estimate. Consequently, the correction acts only on TU:
\begin{equation}
\widehat{\mathrm{TU}}_{\mathrm{JK}}
=
N H(q_N)
-
\frac{N-1}{N}\sum_{i=1}^{N}H(q_{-i}),
\qquad
\widehat{\mathrm{EU}}_{\mathrm{JK}}
=
\widehat{\mathrm{TU}}_{\mathrm{JK}}
-
\widehat{\mathrm{AU}}_N.
\end{equation}
No additional model evaluations are required.

The correction is approximate and need not remove bias for small sample counts, particularly when predictions lie near the simplex boundary. Here, it reduces finite-sample bias when comparing trends across $N$. Because it targets population uncertainty rather than uncertainty of the finite empirical sample, the main-text finite-sample bounds do not directly apply to corrected estimates. Separate corrections to TU and EU also do not guarantee an unbiased estimate of their ratio $\mathrm{EU}/\mathrm{TU}$.

We illustrate the effect of the jackknife correction on EU estimates for different posterior sample counts $N$ in Figures~\ref{fig:N_sweep_jackknife_comparison_cifar} and \ref{fig:N_sweep_jackknife_comparison_mnist}. It is evident that the correction reduces finite-sample bias, as the mean EU curves are mostly flat after correction. When $N<C$ the $\text{EU}$ still increases with $N$ due to the second inequality in Theorem~\ref{th:au_tu_bounds}, which limits the maximum achievable EU to $\log(\min(C, N))$.

\begin{figure}[ht]
\centering
\includegraphics[width=1.0\textwidth]{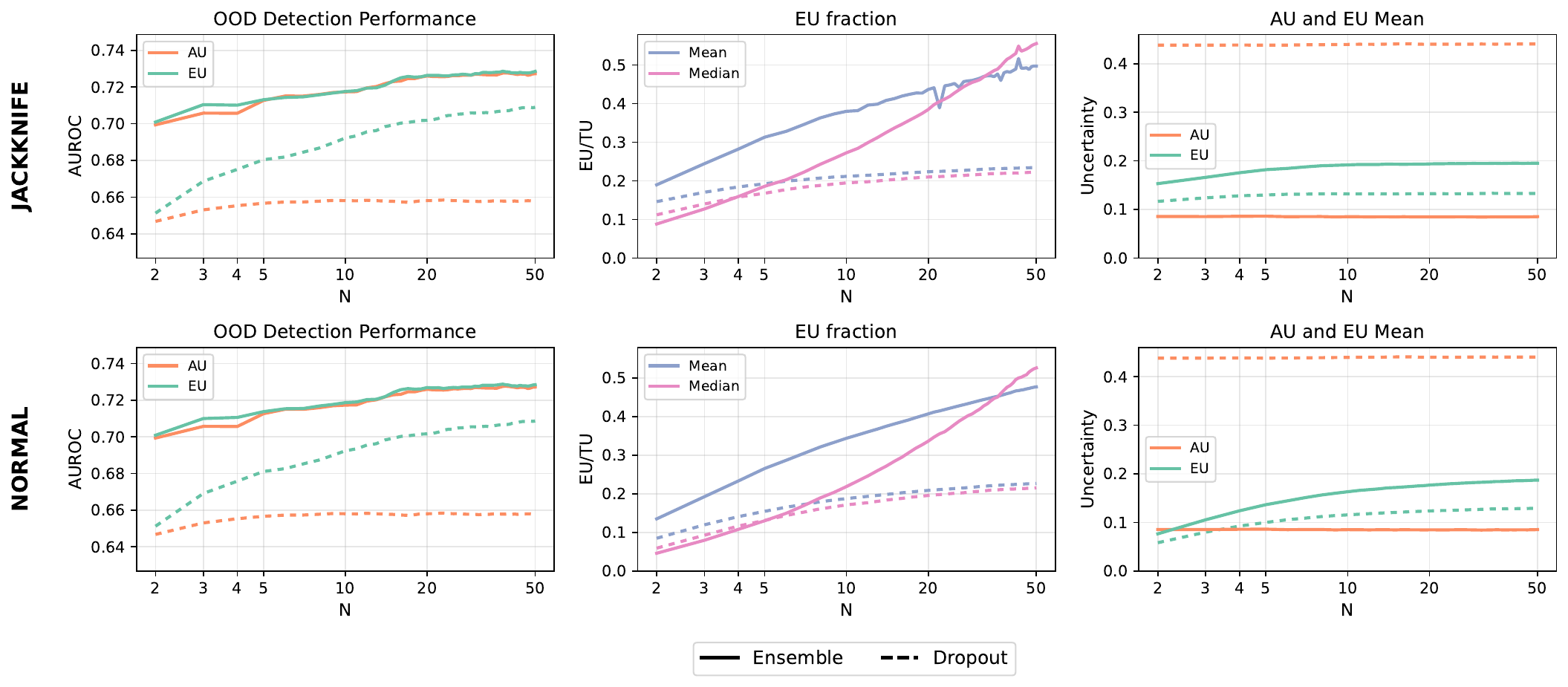}
\caption{Comparison of jackknife-corrected EU estimates for different posterior sample counts $N$ on CIFAR.}
\label{fig:N_sweep_jackknife_comparison_cifar}
\end{figure}

\begin{figure}[ht]
\centering
\includegraphics[width=1.0\textwidth]{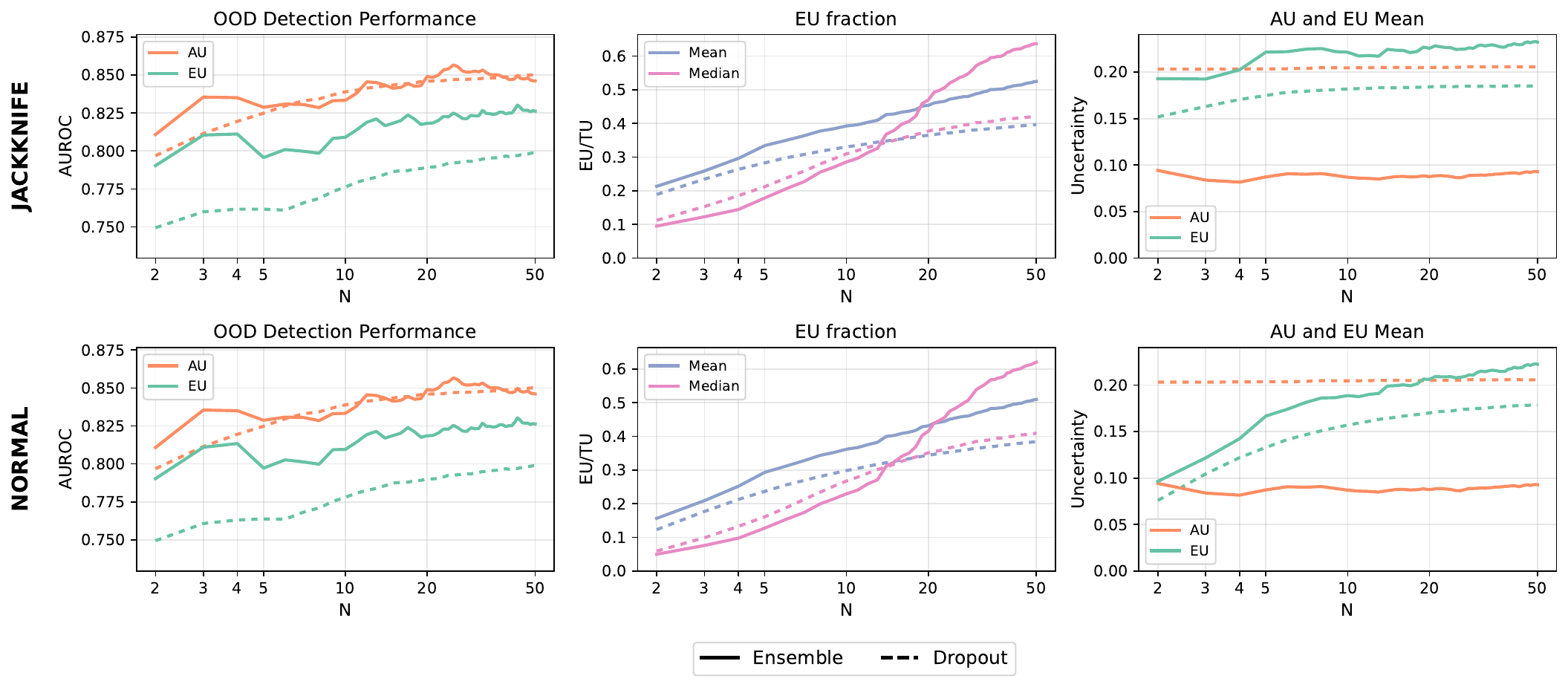}
\caption{Comparison of jackknife-corrected EU estimates for different posterior sample counts $N$ on MNIST.}
\label{fig:N_sweep_jackknife_comparison_mnist}
\end{figure}

\section{Boundary Curve Minimum Lemma}
\label{app:lemma}

\begin{lemma}[Endpoint minimum of boundary curves]
For any boundary curve $(\text{AU}(s),\text{TU}(s))$, the minimum TU occurs at an endpoint.
\end{lemma}

\textbf{Proof.} Let $\alpha\neq\beta$ index the two components of $q$ that vary with $s$. Then
\begin{equation}
    \text{TU}(s) = H(q) = -\sum_{i=1}^{C} q_i \log(q_i) = \underbrace{-\sum_{i\neq \alpha, \beta} q_i \log(q_i)}_{\text{constant in $s$}} - q_\alpha \log(q_\alpha) - q_\beta \log(q_\beta)
\end{equation}
Since $f(x)=-x\log x$ is concave and $q_\alpha,q_\beta$ are linear in $s$, TU is concave in $s$. A concave function on an interval attains its minimum at an endpoint, so the minimum is either $\text{TU}(0)$ or $\text{TU}(1)$. \qed

\end{document}